\documentclass{article}

\usepackage{microtype}
\usepackage{graphicx}
\usepackage{subfigure}
\usepackage{booktabs} 

\usepackage{array,multirow,graphicx}

\usepackage{hyperref}

\usepackage{amsfonts}
\usepackage{amssymb}
\usepackage{amsthm}
\usepackage{amsmath}
\usepackage{multirow}
\usepackage{array}
\usepackage{mathtools}
\usepackage{bm}
\usepackage{graphics}

\makeatletter
\def\Cline#1#2{\@Cline#1#2\@nil}
\def\@Cline#1-#2#3\@nil{%
  \omit
  \@multicnt#1%
  \advance\@multispan\m@ne
  \ifnum\@multicnt=\@ne\@firstofone{&\omit}\fi
  \@multicnt#2%
  \advance\@multicnt-#1%
  \advance\@multispan\@ne
  \leaders\hrule\@height#3\hfill
  \cr}
\makeatother

\usepackage{lastpage}

\theoremstyle{definition}
\newtheorem{definition}{Definition}
\newtheorem{theorem}{Theorem}

\usepackage[accepted]{icml2018}

\begin{document}

\twocolumn[
\icmltitle{Anonymous Walk Embeddings}



\icmlsetsymbol{equal}{*}

\begin{icmlauthorlist}
\icmlauthor{Sergey Ivanov}{skoltech,criteo}
\icmlauthor{Evgeny Burnaev}{skoltech}
\end{icmlauthorlist}

\icmlaffiliation{skoltech}{Skolkovo Institute of Science and Technology, Moscow, Russia}
\icmlaffiliation{criteo}{Criteo Research, Paris, France}

\icmlcorrespondingauthor{Sergey Ivanov}{sergei.ivanov@skolkovotech.ru}

\icmlkeywords{Machine Learning, ICML}

\vskip 0.3in
]



\printAffiliationsAndNotice{}  

\begin{abstract}
The task of representing entire graphs has seen a surge of prominent results, mainly due to learning convolutional neural networks (CNNs) on graph-structured data. While CNNs demonstrate state-of-the-art performance in graph classification task, such methods are supervised and therefore steer away from the original problem of network representation in task-agnostic manner. Here, we coherently propose an approach for embedding entire graphs and show that our feature representations with SVM classifier increase classification accuracy of CNN algorithms and traditional graph kernels. For this we describe a recently discovered graph object, \textit{anonymous walk}, on which we design task-independent algorithms for learning graph representations in explicit and distributed way. Overall, our work represents a new scalable unsupervised learning of state-of-the-art representations of entire graphs. 
\end{abstract}
\section{Introduction}


A wide range of real world applications deal with network analysis and classification tasks. An ease of representing data with graphs makes them very valuable asset in any data mining toolbox; however, the complexity of working with graphs led researchers to seek for new ways of representing and analyzing graphs, of which network embeddings have become broadly popular due to their success in several machine learning areas such as graph classification \cite{survey:embeddings}, visualization \cite{visualizationrelated}, and pattern recognition \cite{imagerelated}.

Essentially, network embeddings are vector representations of graphs that capture  local and global traits and, as a consequence, are more suitable for standard machine learning techniques such as SVM that works on numerical vectors rather than graph structures. Ideally, a practitioner would like to have a \textit{polynomial}-time algorithm that can convert different graphs into different feature vectors. However, such algorithm would be capable of deciding whether two graphs are isomorphic \cite{gartner:hardness}, for which currently only quasipolynomial-time algorithm exists \cite{babai}. Hence, there are fundamental challenges in the design of polynomial-time algorithm for network-to-vector conversion. Instead, a lot of research was devoted to the question of designing network embedding models that are computationally efficient \textit{and} preserve similarity between graphs.



Broadly speaking, network embeddings come from one of the two buckets, either based on engineered graph features or driven by training on graph data. Feature-based methods traditionally appeared in graph kernel setting \cite{rwkernel:10}, where each graph is decomposed into discrete components, distribution of which is used as a vector representation of a graph \cite{rconvolution}. Importantly, general concept of feature-based methods implies ad-hoc knowledge about the data at hand. For example, Random Walk kernel \cite{rwkernel:10} assumes that graph realization originates from the types of random walks a graph has, whereas for Weisfeiler-Lehman (WL) kernel \cite{wlkernel:11} the insight is in subtree patterns of a graph. For high-dimensional graph embeddings feature-based methods produce sparse solution as only few substructures are common across graphs. This is known as \textit{diagonal dominance} \cite{deepgraph}, a situation when a graph representation is only similar to itself, but not to any other graph. 

On the other hand, data-driven approach learns network embeddings by optimizing some form of objective function defined on graph data. Deep Graph Kernels (DGK) \cite{deepgraph}, for example, learns a positive semidefinite matrix that weights the relationship between graph substructures, while Patchy-San (PSCN) \cite{learncnn:16} constructs locally connected neighborhoods for training a convolutional neural network on. Data-driven approach implies learning \textit{distributed} graph representations that have demonstrated promising classification results \cite{learncnn:16, 2dcnn}.


\textbf{Our approach.} We propose to use a natural graph object named \textit{anonymous walk} as a base for learning feature-based and data-driven network embeddings. Recent discovery \cite{anonymouswalks} has shown that anonymous walks provide characteristic graph traits and are capable to reconstruct network proximity of a node \textit{exactly}. In particular, distribution of anonymous walks starting at node $u$ is sufficient for reconstruction of a subgraph induced by all vertices within a fixed distance from $u$; and such distribution uniquely determines underlying Markov processes from $u$, i.e. no two different subgraphs exist having the same distribution of anonymous walks. This implies that two graphs with similar distributions of anonymous walks should be topologically similar. We therefore define feature-based network embeddings on distribution of anonymous walks and show an efficient sampling approach that approximates distributions for large networks.  

To overcome sparsity of feature-based methods, we design a data-driven approach that learns distributed representations on the generated corpus of anonymous walks via backpropagation, in the same vein as neural models in NLP \cite{doc2vec, Bengio:2003}. Considering anonymous walks for the same source node as co-occurring words in the sentence and graph as a collection of such sentences, the hope is that by predicting a target word in a given context of words and a document, the proposed algorithm learns semantic meaning of words and a document. 

To the best of our knowledge, we are the first to introduce anonymous walks in the context of learning network representations and we highlight the following contributions: 



\begin{itemize}
\item Based on the notion of anonymous walk, we propose feature-based network embeddings, for which we describe an efficient sampling procedure to alleviate time complexity of exact computation.
\item By maximizing the likelihood of preserving network proximity of anonymous walks, we propose a scalable algorithm to learn data-driven network embeddings.
\item On widely-used real datasets, we demonstrate that our network embeddings achieve state-of-the-art performance in comparison with other graph kernels and neural networks in graph classification task. 
\end{itemize}

\section{Anonymous Walks}

Random walks are the sequences of nodes, where each new node is selected independently from the set of neighbors of the last node in the sequence. Normally states in a random walk correspond to a label or a global name of a node; however, for reasons described below such states could be unavailable. Yet, recently it has been shown that anonymized version of a random walk can provide a flexible way to reconstruct a network even when global names are absent \cite{anonymouswalks}. We next define a notion of anonymous walk. 

\begin{definition}
	Let $s = (u_1, u_2, \ldots, u_k)$ be an ordered list of elements $u_i \in V$. We define the positional function $pos\text{: } (s, u_i) \mapsto q$ such that for any ordered list $s = (u_1, u_2, \ldots, u_k)$ and an element $u_i \in V$ it returns a list $q = (p_1, p_2, \ldots, p_l)$ of all positions $p_j \in \mathbb{N}$ of $u_i$ occurrences in a list $s$.
\end{definition}

For example, if $s = (a, b, c, b, c)$, then $pos(s, a) = (1)$ as element $a$ appears only on the first position and $pos(s, b) = (2, 4)$. 

\begin{definition}[Anonymous Walk]
	If $w = (v_1, v_2, \ldots, v_k)$ is a random walk, then its corresponding \textit{anonymous walk} is the sequence of integers $a = (f(v_1), f(v_2), \ldots, f(v_k))$, where integer $f(v_i) = \min\limits_{p_j \in pos(w, v_i)}{pos(w, v_i)}$. \newline 
We denote mapping of a random walk $w$ to anonymous walk $a$ by $w \mapsto a$.
\end{definition}

\begin{figure}[h!]
\centering
    \includegraphics[width=1\columnwidth]{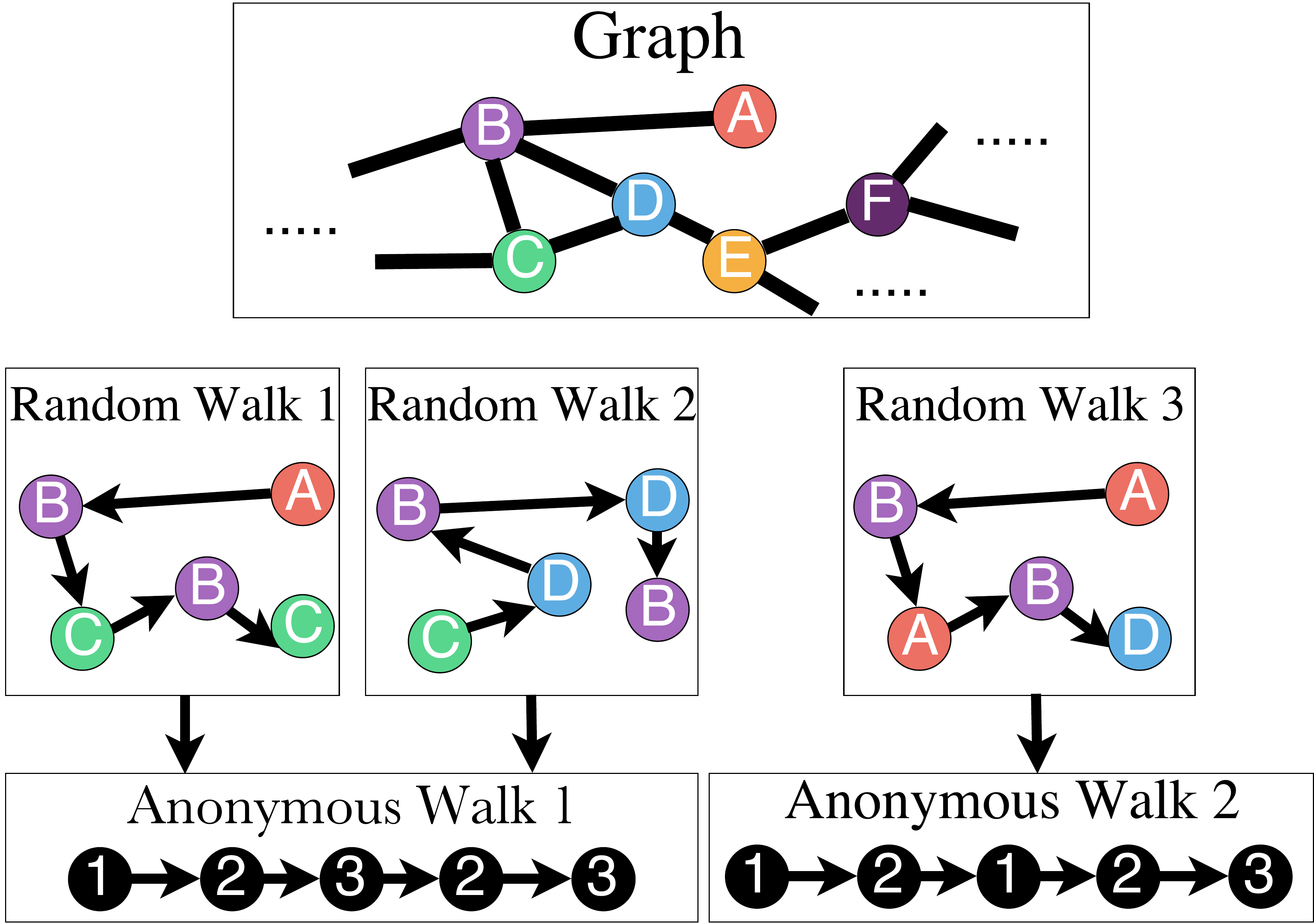}
\caption{An example demonstrating the concept of anonymous walk. Two different random walks 1 and 2 of the graph correspond to the \textit{same} anonymous walk 1. A random walk 3 corresponds to \textit{another} anonymous walk 2.}\label{anonymouswalk}
\end{figure}

For instance, in the graph of Fig. \ref{anonymouswalk} a random walk $a \rightarrow b \rightarrow c \rightarrow b \rightarrow c$ matches anonymous walk $1 \rightarrow 2 \rightarrow 3 \rightarrow 2 \rightarrow 3$. Likewise, another random walk $c \rightarrow d \rightarrow b \rightarrow d \rightarrow b$ also corresponds to anonymous walk $1 \rightarrow 2 \rightarrow 3 \rightarrow 2 \rightarrow 3$. Conversely, another random walk $a \rightarrow b \rightarrow a \rightarrow b \rightarrow d$ corresponds to a different anonymous walk $1 \rightarrow 2 \rightarrow 1 \rightarrow 2 \rightarrow 3$.

Intuitively, states in anonymous walk correspond to the first position of the node in a random walk and their total number equals to the number of distinct nodes in a random walk. Particular name of the state does not matter (so, for example, anonymous walk $1 \rightarrow 2 \rightarrow 3$ would be the same as anonymous walk $3 \rightarrow 1 \rightarrow 2$); however, by agreement, anonymous walks start from $1$ and continue to name new states by incrementing the current “maximum” state in an anonymous walk. 

\textbf{Rationale.} From the perspective of a single node, in the position of an observer, global topology of the network may be hidden deliberately (e.g. social networks often restrict outsiders to examine your friendships) or otherwise (e.g. newly created links in the world wide web may be yet unknown to the search engine). Nevertheless, an observer can, on his own, experiment with the network by starting a random walk from itself, passing the process to its neighbors and recording the observed states in a random walk. As global names of the nodes are not available to an observer, one way to record the states \textit{anonymously} is by describing them by the first occurrence of a node in a random walk. Not only are such records succinct, but it is common to have privacy constraints \cite{privacy} that would not allow to record a full description of nodes. 

Somewhat remarkably, \cite{anonymouswalks} show that for a single node $u$ in a graph $G$, a known distribution $\mathfrak{D}_l$ over anonymous walks of length $l$ is sufficient to reconstruct topology of the ball $B(u, r)$ with the center at $u$ and radius $r$, i.e. the subgraph of graph $G$ induced by all vertices distanced at most $r$ hops from $u$. For the task of learning embeddings, the topology of network is available and thus distribution of anonymous walks $\mathfrak{D}_l$ can be computed precisely. As no two different subgraphs can have the same distribution $\mathfrak{D}_l$, it is useful to generalize distribution of anonymous walks from a single node to the whole network and use it as a feature representation of a graph. This idea paves the way to our feature-based network embeddings.

\section{Algorithms}
We start from discussion of leveraging anonymous walks for learning network embeddings in a feature-based manner. Inspired by empirical results we train an objective function on local neighborhoods of anonymous walks, which further improves results of classification. 
\subsection{AWE: Feature-Based model}

By definition, a weighted directed graph is a tuple $G = (V, E, \Omega)$, where $V = \{v_1, v_2, \ldots, v_n\}$ is a set of $n$ vertices, $E \subseteq V \times V$ is a set of edges, and $\Omega \subset \mathbb{R}$ is a set of edge weights. Given graph $G$ we construct a \textit{random walk graph} $R = (V, E, P)$ such that every edge $e = (u, v)$ has a weight $p_e$ equals to $\omega_e/\sum\limits_{v \in N_{out}(u)}\omega_{(u, v)}$, where $N_{out}(u)$ is the set of out-neighbors of $u$ and $\omega_e \in \Omega$. A random walk $w$ with length $l$ on graph $R$ is a sequence of nodes $u_1, u_2, \ldots, u_{l+1}$, where $u_i \in V$, such that a pair ($u_i, u_{i+1}$) is selected with a probability $p_{(u_i, u_{i+1})}$ in a random walk graph $R$. A probability $p(w)$ of having a random walk $w$ is the total probability of choosing the edges in a random walk, i.e. $p(w) = \prod\limits_{e \in w} p_e$.

According to the Definition \ref{anonymouswalk}, anonymous walk is a random walk, where each state is recorded by its first occurrence index in the random walk. The number of all possible anonymous walks of length $l$ in an arbitrary graph grows exponentially with $l$ (Figure \ref{growth}). Consider an initial node $u$ and a set of all different random walks $W^u_l$ that start from $u$ and have length $l$. These random walks correspond to a set of $\eta$ different anonymous walks $\mathcal{A}_l^u = (a_1^u, a_2^u, \ldots, a_{\eta}^u)$. A probability of seeing anonymous walk $a_i^u$ of length $l$ for a node $u$ is $p(a_i^u) = \sum\limits_{\substack{w \in W^u_l \\ w \mapsto a_i}} p(w)$. Aggregating probabilities across all vertices in a graph and normalizing them by the total number of nodes $N$, we get the probability of choosing anonymous walk $a_i$ in graph $G$: 

\begin{equation*}\label{probability}
p(a_i) = \frac1N\sum\limits_{u \in G} p(a_i^u) =  \frac1N \sum\limits_{u \in G} \sum\limits_{\substack{w \in W^u_l \\ w \mapsto a_i}} p(w).
\end{equation*}

\begin{figure}[h!]
\centering
    \includegraphics[width=1\columnwidth]{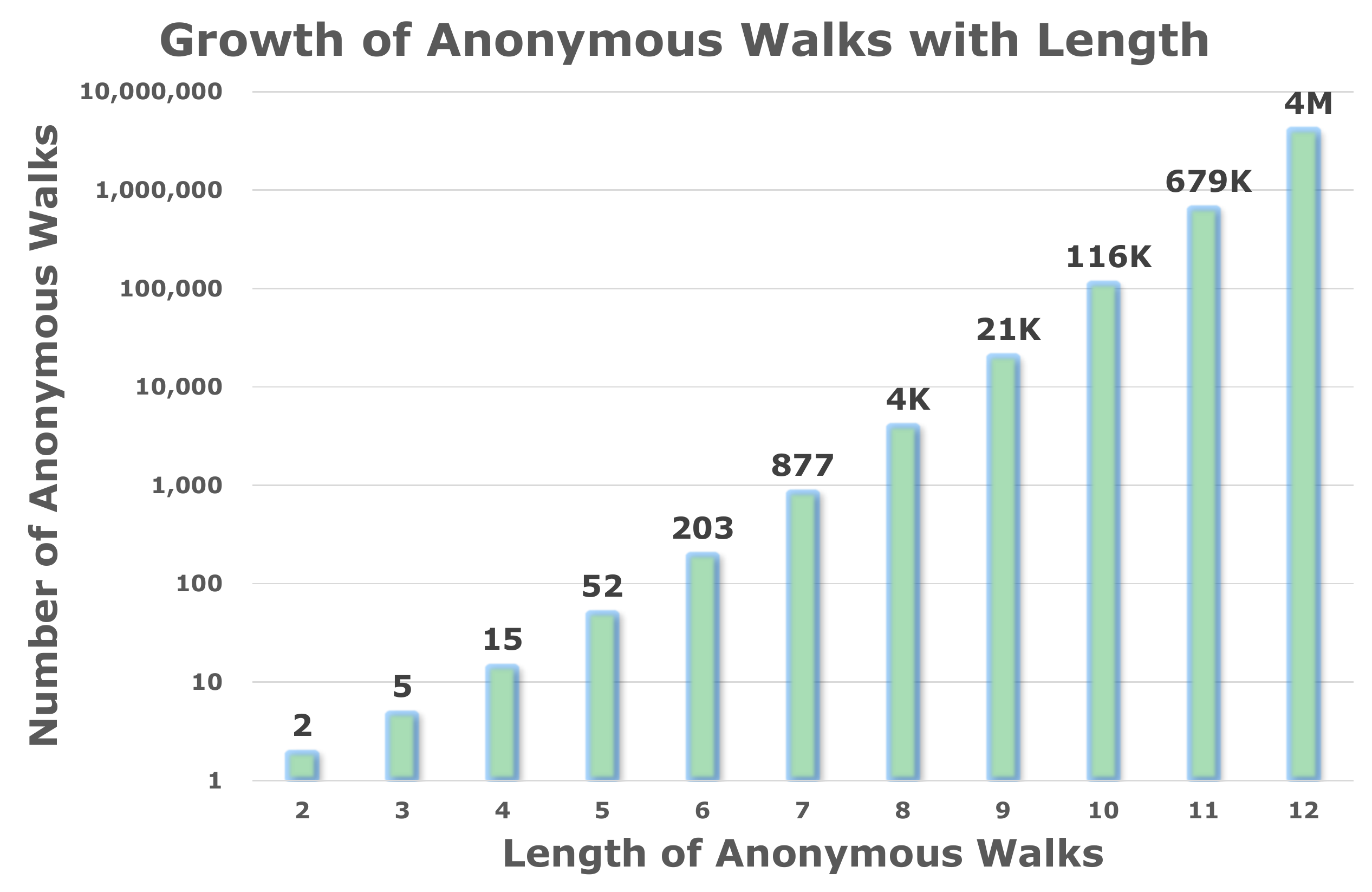}
\caption{$Y$-axis is in log scale. The number of different anonymous walks increases exponentially with length of walks $l$. }\label{growth}
\end{figure}

We are now ready to define network embeddings that we name feature-based anonymous walk embeddings (AWE). 
\begin{definition}[feature-based AWE]
Let $\mathcal{A}_l = (a_1, a_2, \ldots, a_{\eta})$ be the set of all possible anonymous walks of length $l$. \textit{Anonymous walk embedding} of a graph $G$ is the vector $f_G$ of size $\eta$, whose $i$-th component corresponds to a probability $p(a_i)$, of having anonymous walk $a_i$ in a graph $G$:
\begin{equation}\label{awefb}
f_G = (p(a_1), p(a_2), \ldots, p(a_{\eta})).
\end{equation}
\end{definition}

Direct computation of AWE relies on the enumeration of all different random walks in graph $G$, which is shown below to grow exponentially with the number of steps $l$.

\begin{theorem}
	The running time of Anonymous Walk Embeddings (eq. \ref{awefb}) is $\mathcal{O}(nl(d_{in}^{max}(v)\cdot d_{out}^{max}(v))^{l/2})$, where $d_{in/out}^{max}$ is the maximum in/out degree in graph $G$ with $n$ vertices. 
\end{theorem}
\begin{proof}
	Let $k_l$ be the number of random walks of length $l$ in a directed graph. According to \cite{taubig2012} $k_l$ can be bounded by the powers of in- and out-degrees of nodes in $G$:
    \begin{equation*}
    	k_l^2 \le (\sum_{v\in G} d_{in}^l(v))(\sum_{v\in G} d_{out}^l(v)).
    \end{equation*}
Hence, the number of random walks in a graph is at most $n(d_{in}^{max}(v)\cdot d_{out}^{max}(v))^{l/2}$, where $d_{in/out}^{max}$ is the maximum in/out degree. As it requires $\mathcal{O}(l)$ operations to map one random walk of length $l$ to anonymous walk, the theorem follows.
\end{proof}

\textbf{Sampling.} As complete counting of all anonymous walks in a large graph may be infeasible, we describe a sampling approach to approximate the true distribution. In this fashion, we draw independently a set of $m$ random walks and calculate its corresponding empirical distribution of anonymous walks. To guarantee that empirical and actual distributions are close with a given confidence, we set the number $m$ of random walks sufficiently large.

More formally, let $\mathcal{A}_l = (a_1, a_2, \ldots, a_{\eta})$ be the set of all possible anonymous walks of length $l$. For two discrete probability distributions $P$ and $Q$ on set $\mathcal{A}_l$, define $L_1$ distance as:
\begin{equation*}
\Vert P - Q \Vert_1 = \sum\limits_{a_i \in \mathcal{A}}\vert P(a_i) - Q(a_i)\vert
\end{equation*}

For a graph $G$ let $\mathfrak{D}_l$ be the actual distribution of anonymous walks $\mathcal{A}_l$ of length $l$ and let $X^m = (X_1, X_2, \ldots, X_m$) be i.i.d. random variables drawn from $\mathfrak{D}_l$. The empirical distribution $\mathfrak{D}^m$ of the original distribution $\mathfrak{D}_l$ is defined as: 

\begin{equation*}
\mathfrak{D}^m(i) = \frac1m \sum\limits_{X_j \in X^m}[\![ X_j = a_i]\!],
\end{equation*}
where $[\![x]\!] = 1$ if $x$ is true and 0 otherwise.

Then, for all $\varepsilon > 0$ and $\delta \in [0, 1]$ the number of samples $m$ to satisfy $P\{\Vert \mathfrak{D}^m - \mathfrak{D} \Vert_1 \ge \varepsilon\} \le \delta$ equals to (from \cite{graphlet:09}):

\begin{equation}
	m = \left\lceil \frac2{\varepsilon^2}{(\log(2^{\eta} - 2) - \log(\delta))}\right\rceil.
\label{bound}
\end{equation}

For example, there are $\eta = 877$ possible anonymous walks with length $l = 7$ (Figure \ref{growth}). If we set $\varepsilon = 0.5$ and $\delta = 0.05$, then $m = 4888$. If we decrease $\varepsilon = 0.1$ and $\delta = 0.01$, then the number of samples will increase to 122500. 

As transition probabilities for random walks can be preprocessed, sampling of a node in a random walk of length $l$ can be done in $\mathcal{O}(1)$ via alias method. Hence, the overall running time of sampling approach to compute feature-based anonymous walk embeddings is $\mathcal{O}(ml)$.

Our experimental study shows state-of-the-art classification accuracy of feature-based AWE on real datasets. We continue to design data-driven approach that eliminates the sparsity of feature-based embeddings.  

\subsection{AWE: data-driven model}
Our approach for learning network embeddings is analogous to methods for learning paragraph vectors in a text corpus \cite{doc2vec}. In our case, an anonymous walk is a word, a randomly sampled set of anonymous walks starting from the same node is a set of co-occurring words, and a graph is a document.

\textbf{Neighborhoods of anonymous walks.} To leverage the analogy from NLP, we first need to generate a corpus of co-occurring anonymous walks in a graph $G$. We define a neighborhood between two anonymous walks of length $l$ if they share the same source node. This is similar to other methods such as shortest-paths co-occurrence in DGK \cite{deepgraph} and rooted subgraphs neighborhood in graph2vec \cite{graph2vec:algo}, which proved to be successful in empirical studies. Therefore, we iterate over each vertex $u$ in a graph $G$, sampling $T$ random walks $(w^u_1, w^u_2, \ldots, w^u_T)$ that start at node $u$ and map to a sequence of co-occurred anonymous walks $s^u =(a^u_1, a^u_2, \ldots, a^u_T)$, i.e. $w^u_i \mapsto a^u_i$. A collection of all $s^u$ for all vertices $u \in G$ is a corpus of co-occurred anonymous walks in a graph and is analogous to a collection of sentences in a document. 

\textbf{Training.} In this framework, we learn representation vector $\mathrm{d}$ of a graph and anonymous walks matrix $\mathrm{W}$ (see Figure \ref{awe_model}). Vector $\mathrm{d}$ has $1 \times d_g$ size, where $d_g$ is embedding size of a graph. Matrix $\mathrm{W}$ has $\eta \times d_a$ size, where $\eta$ is the number of all possible anonymous walks of length $l$ and $d_a$ is embedding size of anonymous walk. For convenience, we call $\mathrm{d}$ as a document vector and $\mathrm{W}$ as a word matrix. Each graph corresponds to its vector $\mathrm{d}$ and an anonymous walk corresponds to a row in a matrix $\mathrm{W}$. The model tries to predict a target anonymous walk given co-occurring context anonymous walks and a graph. 

Formally, a sequence of co-occurred anonymous walks $s =(a_1, a_2, \ldots, a_T)$ corresponds to vectors $\mathrm{w}_1, \mathrm{w}_2, \ldots, \mathrm{w}_T$ of matrix $\mathrm{W}$, and a graph $G$ corresponds to vector $\mathrm{d}$. We aim to maximize the average log probability:

\begin{equation}
\frac1T \sum\limits_{t = \Delta}^{T - \Delta} \log p(\mathrm{w}_t \vert \mathrm{w}_{t-\Delta}, \ldots, \mathrm{w}_{t+\Delta}, \mathrm{d}),
\label{objective}
\end{equation}

where $\Delta$ is a window size, i.e. number of context words for each target word. Probability in objective (\ref{objective}) is defined via softmax function:
\begin{equation}
p(\mathrm{w}_t \vert \mathrm{w}_{t-\Delta}, \ldots, \mathrm{w}_{t+\Delta}, \mathrm{d}) = \frac{e^{y(\mathrm{w}_t)}}{\sum\limits_{i=1}^{\eta} e^{y(\mathrm{w}_i)}}
\label{softmax}
\end{equation}

Each $y(\mathrm{w}_t)$ is unnormalized log probability for output word $i$: 
\begin{equation*}
y(\mathrm{w}_t) = b + Uh(\mathrm{w}_{t-\Delta}, \ldots, \mathrm{w}_{t+\Delta}, \mathrm{d})
\end{equation*}
where $b \in \mathbb{R}$ and $U \in \mathbb{R}^{d_a + d_g}$ are softmax parameters. Vector $h$ is constructed by first averaging walk vectors $\mathrm{w}_{t-\Delta}, \ldots, \mathrm{w}_{t+\Delta}$ and then concatenating with a graph vector $\mathrm{d}$. The reason is that since anonymous walks are randomly sampled, we average vectors $\mathrm{w}_{t-\Delta}, \ldots, \mathrm{w}_{t+\Delta}$ to compensate for the lack of knowledge on the order of walks; and at the same time, the graph vector $d$ is shared among multiple (context, target) pairs. 

To avoid computation of the sum in softmax equation (\ref{softmax}), which becomes impractical for large sets of anonymous walks, one can use Hierarchical softmax \cite{negativesampling} or NCE loss functions \cite{nce} to speed up training. In our work, we use sampled softmax \cite{sampledsoftmax} that for each training example picks only a fraction of vocabulary according to a chosen sampling function. One can measure distribution of anonymous walks in a graph via means of definition \ref{awefb} and decide on a corresponding sampling function.

At every step of the model, we sample context and target anonymous walks from a graph and compute the gradient error from prediction of target walk and update vectors of context walks and a graph via gradient backpropagation. When given several networks to embed, one can reuse word matrix $\mathrm{W}$ across graphs, thereby sharing previously learned embeddings of walks.

Summarizing, after initialization of matrix $\mathrm{W}$ for all anonymous walks of length $l$ and a graph vector $\mathrm{d}$,  the model repeats the following two steps for all nodes in a graph: 1) for sampled co-occurred anonymous walks the model calculates a loss (Eq. \ref{objective}) of predicting a target walk (one of the sampled anonymous walks) by considering all context walks and a graph; 2) the model updates the vectors of context walks in matrix $\mathrm{W}$ and graph vector $\mathrm{d}$ via gradient backpropagation. One step of the model is depicted in Figure \ref{awe_model}. After using up all sampled corpus, a learned graph vector $\mathrm{d}$ is called \textit{anonymous walk embedding}.

\begin{figure}[t]
\centering
    \includegraphics[width=1\columnwidth]{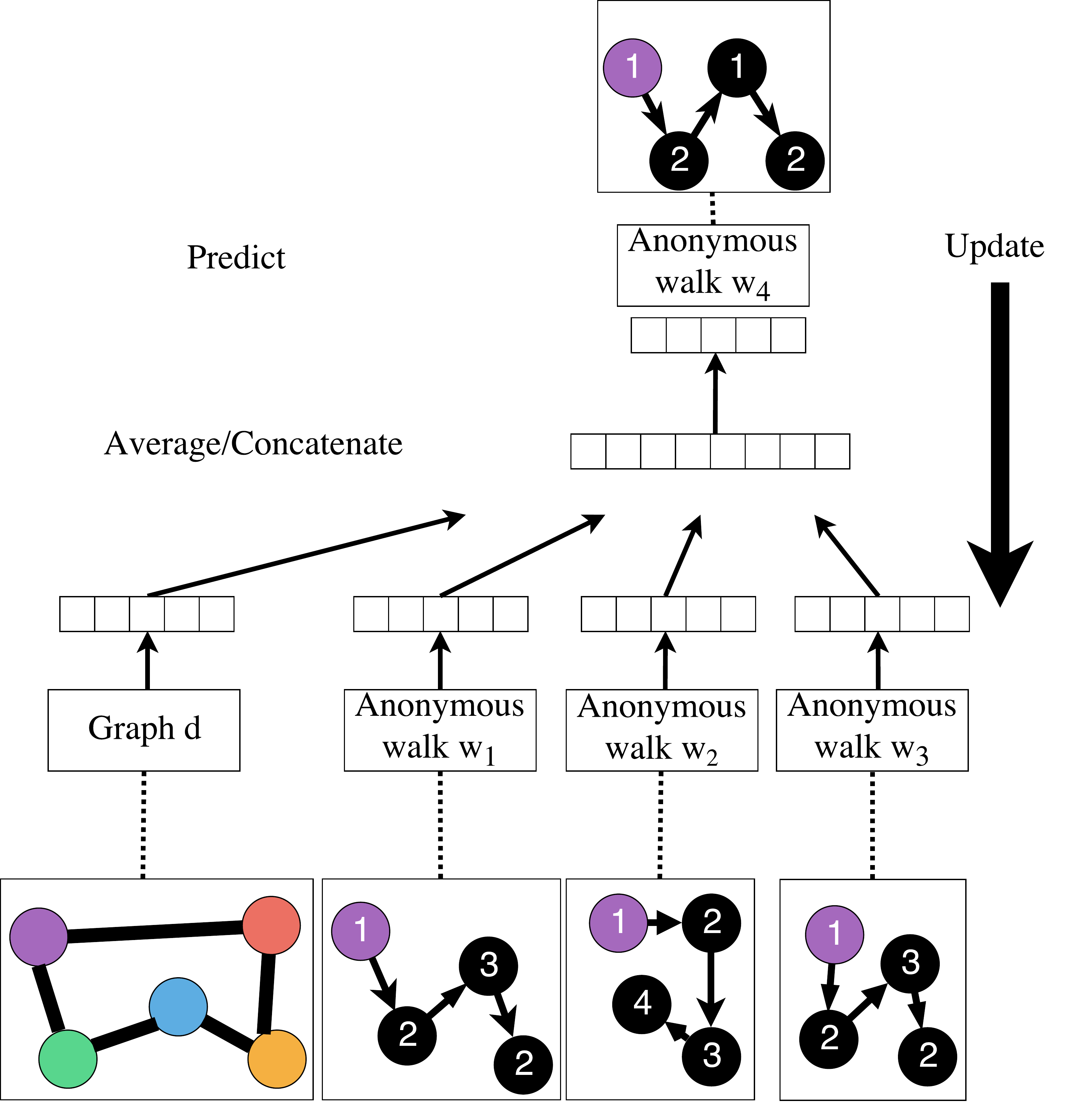}
\caption{A framework for learning data-driven anonymous walk embeddings. Graph is represented by a vector $d$ and anonymous walks are represented by rows of matrix $\mathrm{W}$. All co-occurring anonymous walks start from the same node in a graph. The goal is to predict a target walk $\mathrm{w}_4$ by its surrounding context walks $(\mathrm{w}_1, \mathrm{w}_2, \mathrm{w}_3)$ and a graph vector $\mathrm{d}$. We average embeddings of context walks and then concatenate with a graph vector to predict a target vector. Vectors are updated using stochastic gradient descent on a corpus of sampled anonymous walks.}\label{awe_model}
\end{figure}

\begin{definition}[data-driven AWE]
\textit{Anonymous walk embedding} of a graph $G$ is a vector representation $\mathrm{d}$ learned on a corpus of sampled anonymous walks from a graph $G$. 
\label{awedb}
\end{definition}

So despite the fact that graph and walk vectors are initialized randomly, as an indirect result of predicting a walk in the context of other walks and a graph the model also learns feature representations of networks. Intuitively, a graph vector can be thought as a word with a special meaning: it serves as an overall summary for all anonymous walks in the graph. 

In our experiments, we show how anonymous walk network embeddings can be used in graph classification problem, demonstrating state-of-the-art performance in classification accuracy. 

\section{Graph Classification}
Graph classification is a task to predict a class label of a whole graph and it has found applications in bioinformatics \cite{matching} and malware detection \cite{graph2vec:algo}. In this task, given a series of $N$ graphs $\{G_i\}_{i=1}^N$ and their corresponding labels $\{L_i\}_{i=1}^N$, we are asked to train a model $\textrm{m}$: $G \mapsto L$ that would efficiently classify new graphs. 
Two typical approaches to graph classification problem are (1) supervised learning classification algorithms such as PSCN algorithm \cite{learncnn:16} and (2) graph kernel methods such as WL kernel \cite{wlkernel:11}. As we are  interested in designing task-agnostic network embeddings that do not require labeled data during training, we show how to use anonymous walk embeddings in conjunction with kernel methods to perform classification of new graphs. For this we define a kernel function on two graphs. 

\begin{definition}[Kernel function]
\textit{Kernel function} is a symmetric, positive semidefinite function $k$: $X \times X \mapsto \mathbb{R}^n$ defined for a non-empty set $X$. 
\end{definition}
When $X \subseteq \mathbb{R}^n$, several popular choices of kernel exist \cite{smola:01}:
\begin{itemize}
\item \textbf{Inner product} $k(x, y) = \langle x, y \rangle$, $\forall x, y \in \mathbb{R}^n$,
\item \textbf{Polynomial} $k(x, y) = (\langle x, y \rangle + c)^d$, $\forall x, y \in \mathbb{R}^n$,
\item \textbf{RBF} $k(x, y) = \exp(-\dfrac{\Vert x- y \Vert^2_2}{2\sigma^2})$, $\forall x, y \in \mathbb{R}^n$.
\end{itemize}

With network embeddings, it is then easy to define a kernel function on two graphs: 
\begin{equation}
\label{graphkernel}
K(G_1, G_2) = k(f(G_1), f(G_2)),
\end{equation}
where $f(G_i)$ is an embedding of a graph $G_i$ and $k$: $(x, y) \mapsto \mathbb{R}^n$ is a kernel function.

To train a graph classifier $\mathrm{m}$ one can then construct a square kernel matrix $\mathcal{K}$ for training data $G_1, G_2, \ldots, G_N$ and feed this matrix to a kernelized algorithm such as SVM. Every element of kernel matrix equals to: $\mathcal{K}_{ij} = K(G_i, G_j)$.
For classifying new test instance $G_{\tau}$, one would first compute graph kernels with training instances $(K(G_1, G_{\tau}), K(G_2, G_{\tau}), \ldots, K(G_N, G_{\tau}))$ and provide it to a trained classifier $\textrm{m}$.

In our experiments, we use anonymous walk embeddings to compute kernel matrices and show that kernelized SVM classifier achieves top performance comparing to more complex state-of-the-art models. 
\section{Experiments}
We evaluate our embeddings on the task of graph classification for variety of widely-used datasets.

\textbf{Datasets.} We evaluate performance on two sets of graphs. One set contains \textit{unlabeled } graph data and is related to social networks \cite{deepgraph}. Another set contains graphs with labels on node and/or edges and originates from bioinformatics \cite{wlkernel:11}. Statistics of these ten graph datasets presented in Table \ref{dataset_properties}. 

\textbf{Evaluation.}
We train a multiclass SVM classifier with one-vs-one scheme. We perform a 10-fold cross-validation and for each fold we estimate SVM parameter $C$ from the range [0.001, 0.01, 0.1, 1, 10] using validation set. This process is repeated 10 times and an average accuracy is reported, i.e. the average number of correctly classified test graphs.

\begin{table}[ht]
\caption{Graph datasets used in classification experiments. The columns are: Name of dataset, Number of graphs, Number of classes (maximum number of graphs in a class), Average number of nodes/edges.}
\label{dataset_properties}
\vskip 0.15in
\begin{small}
\resizebox{\columnwidth}{!}{
  \begin{tabular}{| l | p{1cm} | c | p{1.5cm} | p{1cm} | p{1cm} |}
      \hline
      Dataset & Source & Graphs & Classes \newline(Max) & Nodes \newline Avg. & Edges \newline Avg. \\ \hline \hline
      COLLAB & Social & 5000 & 3 (2600) & $74.49$ & $4914.99$ \\ \hline
      IMDB-B & Social & 1000 & 2 (500) & $19.77$ & $193.06$ \\ \hline
      IMDB-M & Social & 1500 & 3 (500) & $13$ & $131.87$ \\ \hline
      RE-B & Social & 2000 & 2 (1000) & $429.61$ & $995.50$ \\ \hline
      RE-M5K & Social & 4999 & 5 (1000) & $508.5$ & $1189.74$ \\ \hline
      RE-M12K & Social & 12000 & 11 (2592) & $391.4$ & $913.78$ \\ \hline
      Enzymes & Bio & 600 & 6 (100) & $32.6$ & $124.3$ \\ \hline
      DD & Bio & 1178 & 2 (691) & $284.31$ & $715.65$ \\ \hline
      Mutag & Bio & 188 & 2 (125) & $17.93$ & $19.79$ \\ \hline
      \hline
  \end{tabular}
}
\end{small}
\vskip -0.1in
\end{table}

\textbf{Competitors.} PSCN is a convolutional neural network algorithm \cite{learncnn:16} with size of receptive field equals to 10. PSCN is the state-of-the-art instance of neural network algorithms, which has achieved strong classification accuracy in many datasets, and we use the best reported accuracy for these algorithms. GK is a graphlet kernel \cite{graphlet:09} and DGK is a deep graphlet kernel \cite{deepgraph} with graphlet size equals to 7. WL is Weisfeiler-Lehman graph kernel algorithm \cite{wlkernel:11} with height of subtree pattern equals to 7. WL proved consistenly strong results comparing to other graph kernels and supervised algorithms. ER is exponential random walk kernel \cite{gartner:hardness} with exponent equals to $0.5$ and kR is $k$-step random walk kernel with $k=3$ \cite{halting:15}.  

\textbf{Setup.}
For feature-based anonymous walk embeddings (Def. \ref{awefb}), we choose length $l$ of walks from the range $[2,3,\ldots, 10]$ and approximate actual distribution of anonymous walks using sampling equation (\ref{bound}) with $\varepsilon = 0.1$ and $\delta = 0.05$. 

For data-driven anonymous walk embeddings (Def. \ref{awedb}), we set length of walks $l=10$ to generate a corpus of co-occurred anonymous walks. We run gradient descent with 100 iterations for 100 epochs with batch size that we vary from the range [100, 500, 1000, 5000, 10000]. Context walks are drawn from a window, which size varies in the range [2, 4, 8, 16]. The embedding size of walks and graphs $d_a$  and $d_g$ equals to 128. Finally, candidate sampling function for softmax equation (\ref{softmax}) chooses uniform or loguniform distribution of sampled classes. 

To perform classification, we compute a kernel matrix, where Inner product, Polynomial, and RBF kernels are tested. For RBF kernel function we choose parameter $\sigma$ from the range $[10^{-5}, 10^{-4}, \ldots, 1, 10]$; for Polynomial function we set $c=0$ and $d=2$. We run the experiments on a machine with Intel(R) Xeon(R) CPU E5-2680 v4 @ 2.40GHz and 32GB RAM\footnote{Code can be found at \url{https://github.com/nd7141/AWE}}. We refer to our algorithms as AWE (DD) and AWE (FB) for data-driven and feature-based approaches correspondingly.

\textbf{Classification results.}
Table \ref{unlabeled_results} presents results on classification accuracy for Social unlabeled datasets. AWE approaches are consistently at the top, sharing top-2 results for all six social datasets, despite being unsupervised approach unlike PSCN. 
At the same time, Table \ref{labeled_results} shows accuracy results for labeled bio datasets. Note that AWE are learned using only topology of the network and not node/edge labels. In this setting, embeddings obtained by AWE (FB) approach achieves competitive performance for the labeled datasets.

\begin{table*}[h]
\caption{Comparison of classification accuracy (mean $\pm$ std., \%) in Social datasets. Top-2 results are in \textbf{bold}. OOM is out-of-memory.}\label{unlabeled_results}
\vskip 0.15in
\begin{center}
\centering
\begin{tabular}{| c| c | c | c | c | c | c | c |}
\hline
& Algorithm & IMDB-M & IMDB-B & COLLAB & RE-B & RE-M5K & RE-M12K \\ \hline \hline
\multirow{3}{*}{DD}
& AWE (DD) & \textbf{51.54 $\pm$ 3.61} & \textbf{74.45 $\pm$ 5.83} & \textbf{73.93 $\pm$ 1.94} & \textbf{87.89 $\pm$ 2.53} & \textbf{50.46 $\pm$ 1.91} & \textbf{39.20 $\pm$ 2.09} \\ \cline{2-8}
& PSCN & 45.23 $\pm$ 2.84 & 71.00 $\pm$ 2.29 & 72.60 $\pm$ 2.15 & \textbf{86.30 $\pm$ 1.58} & 49.10 $\pm$ 0.70 & 41.32 $\pm$ 0.32 \\ \cline{2-8}
& DGK & 44.55 $\pm$ 0.52 & 66.96 $\pm$ 0.56 & 73.09 $\pm$ 0.25 & 78.04 $\pm$ 0.39 & 41.27 $\pm$ 0.18 & 32.22 $\pm$ 0.10 \\ \noalign{\hrule height 2pt}
\multirow{5}{*}{FB}
& AWE (FB) & \textbf{51.58  $\pm$  4.66} & 73.13 $\pm$ 3.28 & 70.99 $\pm$ 1.49 & 82.97 $\pm$ 2.86 & \textbf{54.74 $\pm$ 2.93} & \textbf{41.51 $\pm$ 1.98} \\ \cline{2-8}
& WL & 49.33 $\pm$ 4.75 & \textbf{73.4 $\pm$ 4.63} & \textbf{79.02 $\pm$ 1.77} & 81.1 $\pm$ 1.9 & 49.44 $\pm$ 2.36 & 38.18 $\pm$ 1.3 \\ \cline{2-8}
& GK & 43.89 $\pm$ 0.38 & 65.87 $\pm$ 0.98 & 72.84 $\pm$ 0.28 & 65.87 $\pm$ 0.98 & 41.01 $\pm$ 0.17 & 31.82 $\pm$ 0.08 \\ \cline{2-8}
& ER & OOM & 64.00 $\pm$ 4.93 & OOM & OOM & OOM & OOM \\ \cline{2-8}
& kR & 34.47 $\pm$ 2.42 & 45.8 $\pm$ 3.45 & OOM & OOM & OOM & OOM\\ \hline
\end{tabular}
\end{center}
\vskip -0.1in
\end{table*}

\textbf{Overall observations.} 
\begin{itemize}
\item Tables \ref{unlabeled_results} and \ref{labeled_results} demonstrate that AWE is competitive to supervised state-of-the-art solutions in graph classification task. Importantly, even with simple classifiers such as SVM, AWE increases classification accuracy comparing to other more complex neural network models. Likewise, just comparing graph kernels, we can see that anonymous walks is at the top with tranditional graph objects such as graphlets (GK kernel) or subtree patterns (WL kernel).
\item While feature-based and data-driven approaches are different in nature, the resulted classification accuracy is close across many datasets. As such, only on RE-B dataset data-driven approach has more than 5\% increase in the accuracy. In practice, we found that using feature-based approach for small length $l$ (e.g. $\le 10$) produces competitive results, while data-driven approach works best for large number of iterations and length $l$. 
\item Polynomial and RBF kernel functions bring non-linearity to the classification algorithm and are able to learn more complex classification boundaries. Table \ref{kernel_table} shows that RBF and Polynomial kernels are well suited for feature-based and data-driven models respectively.
\end{itemize}

\begin{table}[h]
\caption{Kernel function comparison in classification task (\%).}\label{kernel_table}
\vskip 0.15in
\begin{center}
\begin{small}
\centering
\begin{tabular}{| c | c | c | c |}
\hline
Algorithm & IMDB-M & COLLAB & RE-B \\ \hline \hline 
AWE (DD) \textit{RBF} & 50.73 & \textbf{73.93} & \textbf{87.89} \\ \hline
AWE (DD) \textit{Inner} & \textbf{51.54} & 73.77 & 84.82 \\ \hline
AWE (DD) \textit{Poly} & 45.32 & 70.45 & 79.35 \\ \noalign{\hrule height 2pt} 

AWE (FB) \textit{RBF} & \textbf{51.58} & \textbf{70.99} & \textbf{82.97} \\ \hline
AWE (FB) \textit{Inner} & 46.45 & 69.60 & 76.83 \\ \hline
AWE (FB) \textit{Poly} & 46.57 & 64.3 & 67.22 \\ \hline
\end{tabular}
\end{small}
\end{center}
\vskip -0.1in
\end{table}

\begin{table}[h]
\caption{Classification accuracy (\%) in labeled Bio datasets.}\label{labeled_results}
\vskip 0.15in
\begin{center}
\begin{small}
\centering
\begin{tabular}{| c | c | c | c | c |}
\hline
Algorithm & Enzymes & DD & Mutag  \\ \hline \hline 
 AWE & 35.77 $\pm$ 5.93 & 71.51 $\pm$ 4.02 & 87.87 $\pm$ 9.76 \\ \hline
 PSCN & $\boldsymbol{-}$ & 77.12 $\pm$ 2.41 & 92.63 $\pm$ 4.21 \\ \hline  
 DGK & 27.08 $\pm$ 0.79 & $\boldsymbol{-}$ & 82.66 $\pm$ 1.45 \\ \hline
 WL & 53.15 $\pm$ 1.14 & 77.95 $\pm$ 0.70 & 80.72 $\pm$ 3.00 \\ \hline
 GK & 32.70 $\pm$ 1.20 & 78.45 $\pm$ 0.26 & 81.58 $\pm$ 2.11 \\ \hline
 ER & 14.97 $\pm$ 0.28 & OOM & 71.89 $\pm$ 0.66 \\ \hline
 kR & 30.01 $\pm$ 1.01 & OOM & 80.05 $\pm$ 1.64 \\ \hline
\end{tabular}
\end{small}
\end{center}
\vskip -0.1in
\end{table}\textbf{}

\textbf{Scalability. }
To test for scalability, we learn network representations using AWE (DD) algorithm for Erdos-Renyi graphs with increasing sizes from [$10$, $10^1$, $10^2$, $10^3$, $10^4$, $3\cdot 10^4$]. For each size we construct 10 Erdos-Renyi graphs with $\mu = np \in [2, 3, 4, 5]$, where $n$ is the number of nodes and $p$ is the probability of having an edge between two arbitrary nodes. In that case, a graph has $m \propto \mu n$ edges. We average time to train AWE (DD) embeddings across 10 graphs for every $n$ and $\mu$. Our setup: size of embeddings equals to 128, batch size equals to 100, window size equals to 100. We run AWE (DD) model for 100 iterations in one epoch. 
In Figure \ref{scalability}, we empirically observe that the model to learn AWE (DD) network representations scales to networks with tens of thousands of nodes and edges and requires no more than a few seconds to map a graph to a vector. 

\begin{figure}[h]
\centering
    \includegraphics[width=1\columnwidth]{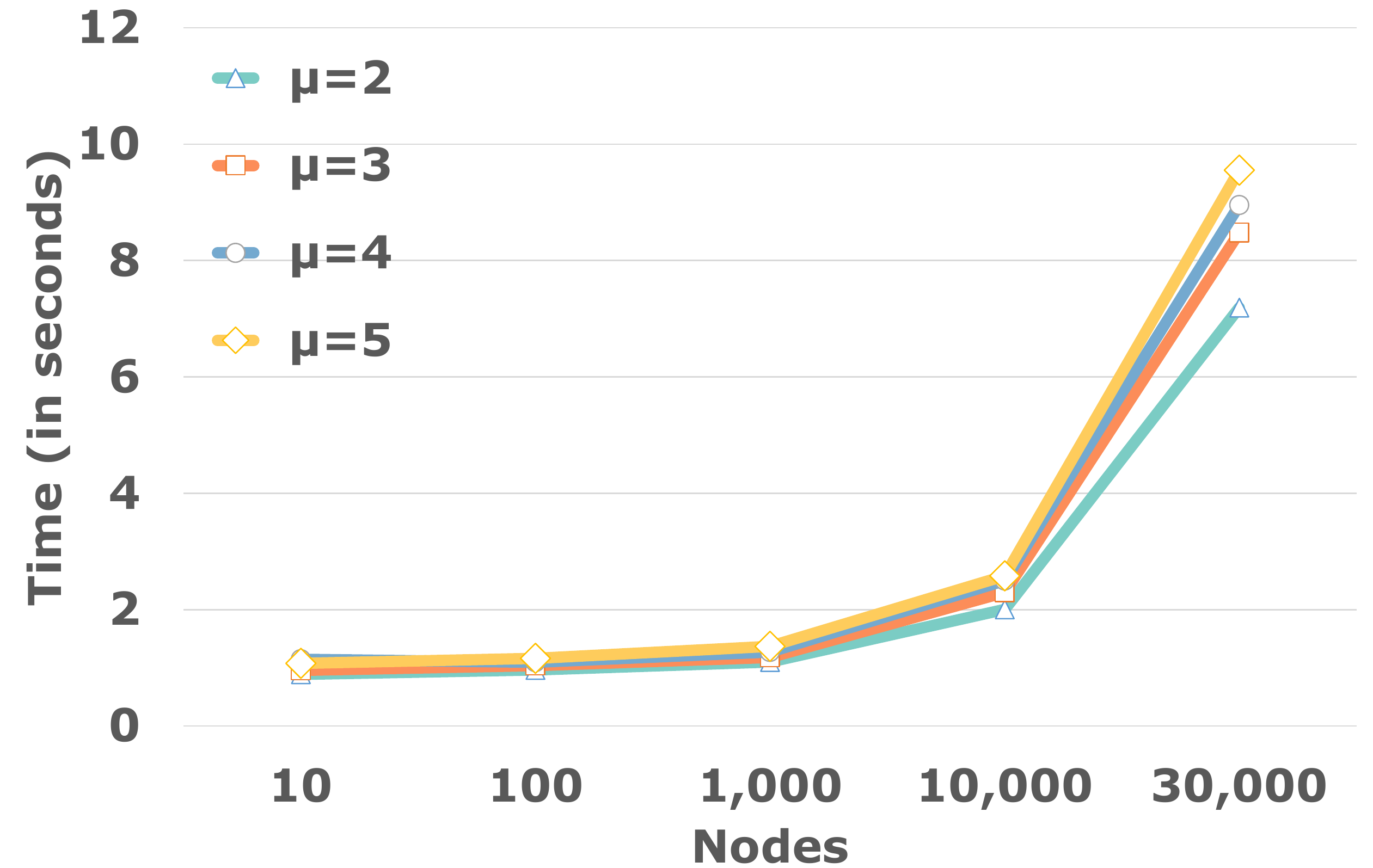}
\caption{Average running time to generate anonymous walk embedding for Erdos-Renyi graphs, with $\mu = np \in [2, 3, 4, 5]$ where $n$ is the number of nodes and $p$ is probability parameter of Erdos-Renyi model. $X$-axis is in log scale.}\label{scalability}
\end{figure}

\textbf{Intuition behind performance.} 
There is a couple of factors that leads anonymous walk embeddings to state-of-the-art performance in graph classification task. First, the use of anonymous walks is backed up by a recent discovery that, under certain condition, distribution of anonymous walks of a single node is sufficient to reconstruct a topology of the ball around a node. Hence, at least on a level of a single node, distribution of anonymous walk serves as a unique representation of subgraphs in a network. Second, data-driven approach reuses hitherto learned embeddings matrix $\mathrm{W}$ in previous iterations for learning embeddings of new graph instances. Therefore one can think of anonymous walks as words that have semantic meaning unified across all graphs. While learning graph embeddings, we simultaneously learn the meaning of different anonymous walks, which provides extra information for our model.


\section{Related Work}
Network representations were first studied in the context of graph kernels \cite{gartner:hardness} and then have become a separate topic that found numerous applications beyond graph classification \cite{survey:embeddings}. Our feature-based embeddings originate from learning distribution on anonymous walks in a graph and is alike to the approach of graph kernels. Embeddings based on graph kernels include Random Walk \cite{gartner:hardness}, Graphlet \cite{graphlet:09}, Weisfeiler-Lehman \cite{wlkernel:11}, Shortest-Path \cite{shortest} decompositions and all can be summarized as an instance of R-convolution framework \cite{rconvolution}. 

Distributed representations have become trendy after significant achievements in NLP applications \cite{word2vec,negativesampling}. Our data-driven network embeddings stem from paragraph-vector distributed-memory model \cite{doc2vec} that has become successful in learning document representations. Other related approaches include Deep Graph Kernel \cite{deepgraph} that learns a matrix for graph kernel that encodes relationship between substructures; PSCN \cite{learncnn:16} and 2D CNN \cite{2dcnn} algorithms that learn convolutional neural networks on graphs; graph2vec \cite{graph2vec:algo} learns network embeddings by extracting rooted subgraphs and training on skipgram negative sampling model \cite{negativesampling}; FGSD \cite{FGSD} that constructs feature vector from the histogram of the multiset of node pairwise distances. \cite{survey:embeddings} provides a more comprehensive list of graph embeddings. Besides this, there is a list of aggregation techniques of node embeddings for the purpose of graph classification \cite{leskovec:representation}.

\section{Conclusion}
We described two unsupervised algorithms to compute network vector representations using anonymous walks. In the first approach, we use distribution of anonymous walks as a network embedding. As the exact calculation of network embeddings can be expensive we demonstrate how one can sample walks in a graph to approximate actual distribution with a given confidence. Next, we show how one can learn distributed graph representations in a data-driven manner, similar to learning paragraph vectors in NLP. 

In our experiments, we show that our network embeddings even with simple SVM classifier achieve increase in classification accuracy comparing to state-of-the-art supervised neural network methods and graph kernels. This demonstrates that representation of your data can be more promising subject to study than the type and architecture of your predictive model. 

Although the focus of this work was in representation of networks, AWE algorithm can be used to learn node, edge, or any subgraph representations by replacing graph vector with a corresponding subgraph vector. In all graph and subgraph representations, we expect data-driven approach to be a strong alternative to feature-based methods. 

\section{Acknowledgement}
This work was supported by the Ministry of Education and Science of the Russian Federation (Grant no. 14.756.31.0001) and by the Skoltech NGP Program No. 1-NGP-1567 “Simulation and Transfer Learning for Deep 3D Geometric Data Analysis” (a Skoltech-MIT joint project).

\bibliography{example_paper}
\bibliographystyle{icml2018}

\end{document}